\documentclass[letterpaper, 10 pt, conference]{ieeeconf}  

\makeatletter

\let\proof\@undefined
\let\endproof\@undefined
\makeatother

\IEEEoverridecommandlockouts                              

\usepackage{graphics} 
\usepackage{epsfig} 
\usepackage{amsmath,amsthm} 
\usepackage{amssymb}  
\usepackage{algorithm,algorithmicx}
\usepackage[noend]{algpseudocode}
\usepackage{mathtools}
\usepackage{xcolor}
\usepackage{setspace}
\usepackage{algpseudocode}
\usepackage{listings}
\usepackage{amsfonts} 
\usepackage{listings}

\usepackage{xcolor}
\usepackage{subfigure}

\usepackage{graphicx}
\usepackage[caption=false]{subfig}
\usepackage{booktabs}
\usepackage{makecell} 
\usepackage{subfig}
\usepackage{url}
\usepackage{cite}

\newtheorem{theorem}{Theorem}[section]

\newtheorem{problem}{Problem}

\renewcommand{\baselinestretch}{.987}

\newcommand{\G}{\mathcal{G}}

\newcommand{\obs}{\mathcal{W}_{\text{obs}}}

\newcommand{\GP}{G^{\mathcal{P}}}

\title{\LARGE \bf
Real-Time Navigation for Autonomous\\Surface Vehicles In Ice-Covered Waters
}

\author{Rodrigue de Schaetzen$^{1}$, Alexander Botros$^{1}$, Robert Gash$^{2}$, Kevin Murrant$^{2}$, Stephen L.\ Smith$^{1}$%
\thanks{This work is supported in part by the Natural Sciences and Engineering Research Council of Canada (NSERC) and by the National Research Council Canada (NRC).} %
\thanks{$^{1}$Department of Electrical and Computer Engineering, University of Waterloo, Waterloo, ON N2L 3G1, Canada (e-mail: \protect\url{{rdeschae, alexander.botros, stephen.smith}@uwaterloo.ca})}%
\thanks{$^{2}$National Research Council Canada (e-mail: \protect\url{{robert.gash, kevin.murrant}@nrc-cnrc.gc.ca})}%
}

\begin{document}

\maketitle
\thispagestyle{empty}
\pagestyle{empty}

\begin{abstract}
Vessel transit in ice-covered waters poses unique challenges in safe and efficient motion planning. When the concentration of ice is high, it may not be possible to find collision-free trajectories. Instead, ice can be pushed out of the way if it is small or if contact occurs near the edge of the ice. In this work, we propose a real-time navigation framework that minimizes collisions with ice and distance travelled by the vessel. We exploit a lattice-based planner with a cost that captures the ship interaction with ice. To address the dynamic nature of the environment, we plan motion in a receding horizon manner based on updated vessel and ice state information. Further, we present a novel planning heuristic for evaluating the cost-to-go, which is applicable to navigation in a channel without a fixed goal location. The performance of our planner is evaluated across several levels of ice concentration both in simulated and in real-world experiments. 
\end{abstract}

\section{INTRODUCTION}
Recent successes in deploying autonomous ship navigation systems demonstrate that ship autonomy can improve marine safety and travel efficiency \cite{bergman2020optimization}. These benefits are especially appealing to crews tasked with high-risk missions such as arctic navigation through ice-covered waters~\cite{canada_2019, goerlandt2017analysis}. This setting typically features a significantly higher concentration of obstacles than most autonomous navigation applications~\cite{murrant2021dynamic}. In addition, the obstacles are dynamic in nature, moving in response to actions taken by the ship and contributing to the overall complexity of the problem.

In this paper, we address the problem of planning motion for an autonomous surface vehicle (ASV) operating in icy waters. We assume that the ASV is built for arctic or sub-arctic conditions but has limited ice-breaking capabilities. Depending on the ice conditions, these ASVs may need to be escorted by an icebreaker which creates a narrow channel containing a high concentration of ice of widely different sizes (e.g., navigation through the St. Lawrence Seaway during winter).  Our proposed framework operates by planning a reference path to be tracked by the ASV, and then replanning in a receding horizon fashion using updated ice information.
\begin{figure}[t]
    \centering
    \includegraphics[width=0.44\textwidth]{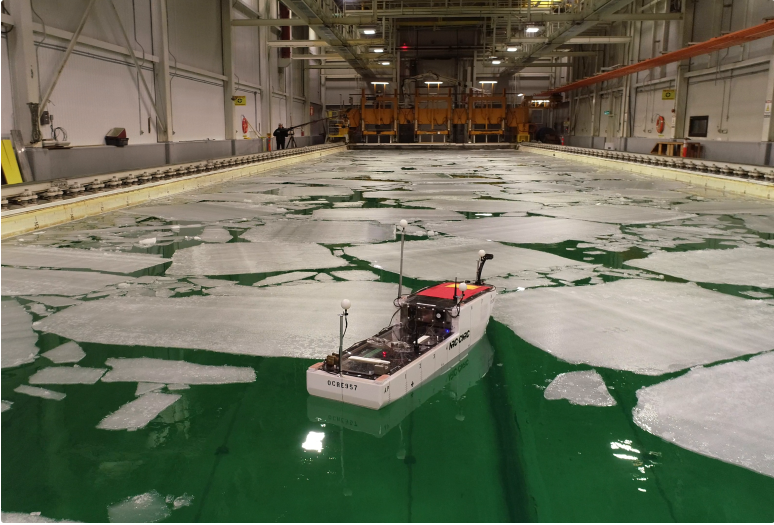}
    \caption{Model vessel being tested in the $12\ \text{m}  \times  76\ \text{m}$ National Research Council Canada (NRC) ice tank in St. John's, Newfoundland and Labrador.}
    \label{fig:icetank}
\end{figure}   

Existing work on path planning for ASVs predominantly focuses on planning collision-free paths~\cite{bergman2020optimization, chiang2018colreg, shan2020receding, choi2015arctic, aksakalli2017optimal}. For the conditions described above --- and illustrated in Figure \ref{fig:icetank} --- such paths may not exist. Thus, rather than avoiding ice, the objective should be to minimize a cost function that captures the effect of a collision with an ice piece (for example, the energy transferred during collision).  In this work we propose a planning framework that utilizes such a cost function.

\emph{Contributions:}  For ships operating in an ice field, we propose a local path planner which we incorporate into a real-time navigation system. The underlying methodology we adopt is the familiar A* path planner over a graph of motion primitives \cite{pivtoraiko2009differentially}. However, to apply this methodology to ship navigation in ice, several key challenges are addressed. We adapt an existing collision model for ship-ice interactions \cite{daley2014gpu} to show that a cost function modelled on \emph{kinetic energy} efficiently penalizes head-on collisions with large ice pieces. Further, since the goal of our planner is forward progress through an ice channel, we propose an admissible closed-form heuristic where the objective is to reach a line segment as opposed to a single configuration. Finally, we demonstrate the efficacy of our approach in ice fields through experiments --- both in simulation and in a physical $12\ \text{m} \times 76\ \text{m}$ ice tank. 

\emph{Related Work:}  Navigation among movable obstacles (NAMO) is considered in \cite{stilman_planning_2008, hai-ning_wu_navigation_2010}. The authors address the problem of navigating a robot through an environment where obstacles can be manipulated by the robot. The NAMO problem is similar to the problem of navigating an ASV through an ice field in that obstacles are movable in both settings. However, unlike ASV navigation the obstacles considered in the NAMO problem typically do not interact with each other once manipulated. Of perhaps greater consequence, obstacles inflict no damage to the robot. 

Extensive work has been done in path planning, tracking control, and collision avoidance algorithms for ASVs. For many of these works, the obstacles considered are either other vessels or static surrounding environments~\cite{zhuang2011motion, kuwata2013safe, chiang2018colreg} like harbors~\cite{bergman2020optimization} and urban waterways~\cite{shan2020receding}. The approaches proposed in these works are effective in their settings, but assume that collisions are forbidden. Therefore, they are not easily generalized to path planning in highly congested environments where collisions are unavoidable.

In path planning for ships, related to our work, the authors in \cite{choi2015arctic} use aggregated statistics of the ice conditions retrieved from satellite imaging for their uncertainty-based route planner. Their work plans paths -- consisting of a set of waypoints -- on the scale of several thousand kilometers which require a local planner (e.g., the one presented in this paper) for real-time autonomous navigation. Local path planning is considered in \cite{hsieh2021sea} where the authors propose a motion planner using bidirectional RRT and data gathered from marine radar imaging. Similarly, in \cite{gash2020machine} the authors construct a graph using a morphological skeleton from a post processed overhead snapshot of an ice field. They then employ A* to compute a path in the resulting graph.  The techniques proposed in \cite{hsieh2021sea, gash2020machine} empirically work well in low-concentration ice fields. However, planned paths are piece-wise straight lines and do not take into account the ASVs' dynamics. Further, the authors assume that computed paths can and must be collision-free. Minimal turning radius constraints are considered in~\cite{aksakalli2017optimal}, but the authors still consider collision-free paths. 

\section{Problem Formulation}
The problem we address is navigating a ship through a cluttered ice channel (as in Figure \ref{fig:problem_form}(a)) while minimizing the effects of ship-ice collisions on the ship---namely the energy transferred in the collisions. We treat the water surface on which a ship moves as a 2D surface $\mathcal{W}\subseteq\mathbb{R}^2$. 

Our objective is forward progress along a channel $\mathcal{C}\subseteq \mathcal{W}$ which we model as a rectangle with length parallel to the $y$-axis as in \ref{fig:problem_form}(a). This model does not limit the applicability of our method since long, curved channels can be partitioned into smaller rectangles as in Figure \ref{fig:problem_form}(b). The objective is to reach a goal line segment $\G \subset \mathcal{C}$ with constant $y$-value as illustrated in Figure \ref{fig:problem_form}(a). Given $m$ ice floes in $\mathcal{C}$, we treat each floe as an obstacle $O_i\subset \mathcal{W}, \ i=1,\dots,m$ which we group into a set of obstacles $\mathcal{W}_{\text{obs}}=\{O_1,\dots,O_m\}$. We assume that the location of each obstacle in $\obs$ can be estimated at any time via an onboard vision system.

\begin{figure}[t]
    \centering
    \includegraphics[width=0.39\textwidth]{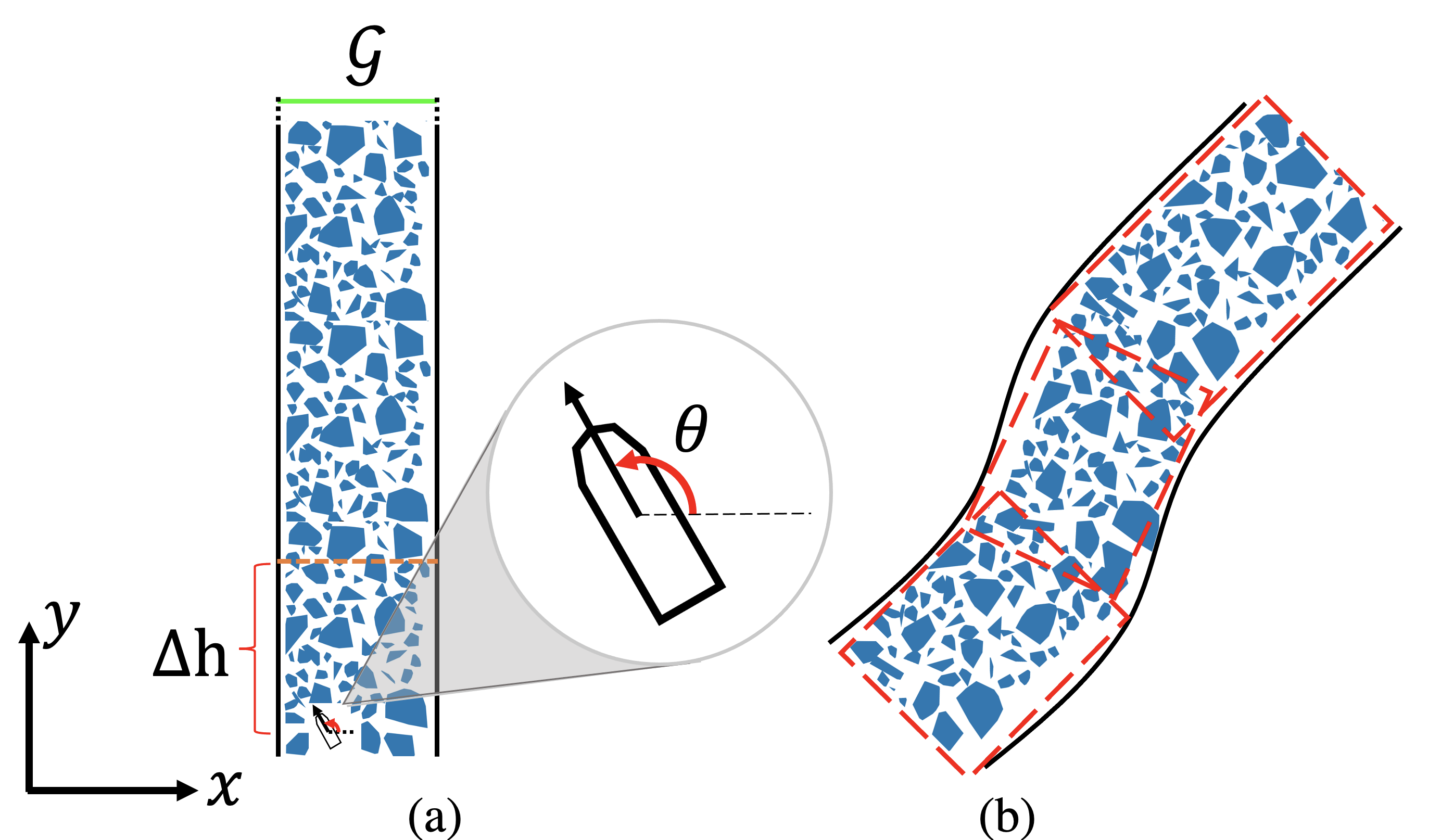}
    \caption{(a) Depiction of the navigation problem of interest for a rectangular channel with ice (blue), ASV (black), and goal region (green line). (b) Generalization of problem to curved channel.}
    \label{fig:problem_form}
\end{figure}
To navigate a ship through an ice field, three components are required: a \emph{reference path} $\pi(s)$ parameterized by arc-length $s$ along the path, a \emph{velocity profile} $v$ that describes how the path is traversed through time, and a \emph{controller} for tracking the path at the desired velocity. The tuple $(\pi, v)$ is called a \emph{trajectory}. To capture the effects of ice collision, we utilize a cost function $u$ that penalizes the path length (final arc length) $s_f$ and a \emph{collision cost}:
\begin{equation}
\label{eq:cost_fn}
u(\pi,v) = s_f + \underbrace{\alpha\int_0^{s_f}c_{\text{obs}} \big(\pi(s), \obs,v(s)\big) \ ds}_{\text{collision cost}}.
\end{equation}
Where $c_{\text{obs}} \geq 0$ is described in detail in Section~\ref{sec:approach} and $\alpha\geq 0$ is a tuning parameter. Thus, we seek to solve the following problem.
\begin{problem}
\label{prob:main_prob}
Given a start position of the ship, a goal line $\mathcal{G}$, and obstacles $\obs$, compute a trajectory $(\pi, v)$ from the start to $\mathcal{G}$ that minimizes $u$.
\end{problem}
In what follows we propose a solution to this problem and describe how it is incorporated into a navigation system for real-time (re-)planning.

\section{Navigation Framework}
\label{sec:approach}
In this section, we discuss our navigation framework and approach to solving Problem \ref{prob:main_prob}. To account for the large and evolving environment, we frequently re-plan a reference path and velocity profile for a controller in real-time. Replanning is done over a moving horizon as illustrated by the part of the channel lying below the orange dotted line in Figure \ref{fig:problem_form}(a). 

\begin{algorithm}[h]
\caption{Receding Horizon Navigation Framework}\label{alg:main}
 \hspace*{\algorithmicindent} \textbf{Input:} $\mathcal{G}, \Delta h, \Delta t$
\label{alg:GeneralPlanner}
\begin{algorithmic}[1]
\While{\texttt{True}}
\State $\obs\gets$ obstacles from onboard sensors 
\State $P, v_S\gets $ current ship pose and speed 
\State Set intermediate goal $\mathcal{G}^i$ a distance $\Delta h$ ahead of $P$
\If{ship pose $P$ has not crossed $\mathcal{G}$}
\State $\pi \gets $ path from $P$ to $\mathcal{G}^i$ with \emph{static} $\obs$, $v_S$
\State $v_{\text{nom}}\gets $ velocity profile for $\pi$
\State Send $(\pi, v_{\text{nom}})$ to controller to track for time $\Delta t$
\Else
\State Exit loop
\EndIf
\EndWhile
\end{algorithmic}
\end{algorithm}

The navigation framework is summarized in Algorithm \ref{alg:GeneralPlanner}, which takes as input a final goal line segment $\mathcal{G}$, receding horizon parameter $\Delta h$, and control tracking parameter $\Delta t$. At the start of each iteration we obtain updated ice data $\obs$ from onboard sensors (Line 2) along with the ship's current position in $\mathbb{R}^2$ and orientation (called a \emph{pose} $P$) and speed  (Line 3). We compute an intermediate goal line $\G^i$ a distance $\Delta h$ ahead of the ship (Line 4; see Figure \ref{fig:problem_form}(a)). If the ship has not yet reached the final goal $\mathcal{G}$ (Line 5), we run our proposed planning algorithm (see Section~\ref{sec:state_lattice}), returning a path from $P$ to $\mathcal{G}^i$ (Line 6). A nominal velocity profile is then computed for the path (Line 7).  In our planner, we use a constant nominal speed based on the concentration of ice, following established guidelines such as \cite{canada_2019}. The planned trajectory $(\pi,v_{\text{nom}})$ is sent to a controller (Line 8) which tracks it for a time $\Delta t$ before the process repeats. We employ a Dynamic Positioning (DP) controller described in Section \ref{control section}. The following sections describe how paths $\pi$ are computed.

Note that at each iteration, ice is treated as static when computing a trajectory $(\pi, v_{\text{nom}})$ but is updated frequently to account for collisions and ice movement (Line 2). This approach maintains low planning run-time for each iteration while still accounting for complex ice movement across all iterations (this is explored in the evaluation in Section \ref{resultssection}).

\subsection{Path Planning using a State Lattice}
\label{sec:state_lattice}
To solve Problem~\ref{prob:main_prob}, we use a form of lattice planner~\cite{pivtoraiko2009differentially}, where paths are generated using a finite set of motion primitives.

\subsubsection{Model}
\label{sec:Model}
For the purposes of path planning, we treat the ship as a 2D rigid body with three degrees of freedom~\cite{fossen2011handbook}:  planar position $(x,y)\in\mathbb{R}^2$ and heading $\theta\in[0, 2\pi)$. We refer the tuple $(x,y,\theta)$ as a \emph{configuration}.  We adopt a unicycle model commonly used in marine navigation \cite{kim2014angular, liang2020path, caillau2019zermelo}: $\dot{x}(s)=\cos(\theta), \dot{y}(s)=\sin(\theta), \dot{\theta}(s) = \kappa, \ |\kappa|\leq \kappa_{\max}$,
where derivatives $\dot{(\cdot)}$ are taken with respect to arc length $s$, $\kappa$ is the \emph{curvature}, and $\kappa_{\text{max}}$ is a maximum curvature dictated by the physical limits of the ship. A path $\pi(s)$ is \emph{feasible} if it adheres to the unicycle model above. 

\subsubsection{Lattice planning}
In lattice-based path planning, the set of all configurations is discretized into a regularly repeating grid called a \emph{lattice} $L$.  Configurations in the lattice are called \emph{vertices}. A set (called a \emph{control set}) of feasible paths between lattice vertices (called \emph{motion primitives}) is pre-computed offline and used as an action set during an online search. The key observation in lattice-based path planning is that motion primitives may be rotated, translated, and concatenated to form complex paths (observe that motion primitives \emph{are} paths between lattice vertices). 

For a control set $\mathcal{P}$, we can construct a graph $\GP=(L, E, \tilde{u})$ the vertices of which are the lattice vertices $L$ while edges are pairs of vertices $(p^1, p^2)$ such that there exists a motion primitive $\pi\in\mathcal{P}$ that can be rotated and translated so that $\pi(0)=p^1, \pi(s_f)=p^2$. Given the ship's current speed $v_S$ (Line 3 of Algorithm \ref{alg:GeneralPlanner}) and $\obs$ (Line 2), the cost of an edge $(p^1, p^2)$ with associated motion primitive $\pi$ is given by $\tilde{u}((p^1, p^2)) = u(\pi, v_S)$ from \eqref{eq:cost_fn}. The problem of computing a feasible path between lattice vertices reduces to computing a cost-minimizing path in $\GP$. 

We generate a state lattice $L$ by discretizing the plane $\mathbb{R}^2$ into a uniform grid and the headings $\theta$ into uniformly spaced angles around the unit circle. To define our motion primitives, we compute shortest paths between $(x,y,\theta)$ configurations for the unicycle model, known as Dubin's paths~\cite{dubins1957curves}. These paths are comprised of sequences of straight lines and circular arcs of radius $r_{\min}=\kappa_{\max}^{-1}$.  The control set $\mathcal{P}$ is generated using the method proposed in \cite{botros2021multi} which computes a control set that balances a tradeoff between $|\mathcal{P}|$ and path optimality.  

We use the lattice framework to plan a path $\pi$ (Line 6) from the ship's current pose $P$ (Line 3) to the current goal $\G^i$ (Line 4) with the lattice origin aligned with $P$. This is accomplished by searching over the graph $\GP$ using the A* search algorithm \cite{hart1968formal} with a line segment as the objective as opposed to a single configuration. As a final step in Line 6, we run a smoothing algorithm on $\pi$ introduced in \cite{botros2021multi} as a post-processing step to remove excessive oscillations. In Section \eqref{eq:cost_fn} we finalize the cost function $u$ from \eqref{eq:cost_fn} while in Section \ref{sec:heuristic} we present a heuristic to improve the performance of A*.

\subsection{Cost Function}
\label{sec:cost_function}
We describe the cost function $u$ from \eqref{eq:cost_fn} given a feasible path $\pi$, obstacles $\obs$, and the ship's current speed $v_S$. In detail, we use the notion of \emph{kinetic energy} to derive a function $c_{\text{obs}}(\pi(s), \obs, v_S)$ that captures the severity of ship-ice collisions in terms of the energy transferred during collision.

We restrict our attention to convex, polygon-shaped ice and assume uniform ice density and thickness in the individual ice floes \cite{daley2014gpu, canada_2020}.  Ship-ice collisions are modeled and empirically validated in \cite{daley2014gpu}, and we adopt a simplified version of their 2D inelastic-collision method to model ship-ice collisions. Specifically, we treat the ship and ice as disks rather than polygons and the ice as static prior to collision. This simplified model lets us efficiently capture sufficient detail in our collision cost and similar approximations have been made in existing work \cite{popov1969strength, daley1999energy}.

With this collision model, the change in kinetic energy $\Delta K_{\text{sys}}$ in the system is given by
\begin{equation}
\label{eq:DeltaKsys}
    \Delta K_{\text{sys}} = \frac{1}{2}M_{\text{eq}}V_{\text{eq}}^2,
\end{equation}
where $M_{\text{eq}}$ and $V_{\text{eq}}$ are the effective inertial mass and velocity, respectively, at the moment of collision \cite{daley2014gpu}. For the case of two disk-shaped bodies, the effective mass $M_{\text{eq}}$ is constant and is defined as:
\begin{equation}
    M_{\text{eq}} = \frac{m_Sm_I}{m_S + m_I},
\end{equation}
where $m_S, m_I$ are the ship and ice masses, respectively. We assume that $m_S$ is known and $m_I$ can be calculated as the product of the area, thickness, and density \cite{daley2014gpu}. The effective velocity $V_{\text{eq}}$ is given by \cite{daley2014gpu}:
\begin{equation}
    V_{\text{eq}} = v_S\cos(\theta),
\end{equation}
where $v_S$ is measured in Line 3 of Algorithm \ref{alg:GeneralPlanner} and $\theta$ is the angle between the ship's heading and the collision normal $\vec{n}$ (Figure \ref{fig:collision}).  Next, we isolate the change in kinetic energy of the ship $\Delta K_S$. 

The goal of our collision cost will be to minimize the ship's kinetic energy loss from ship-ice collisions. Since the ice is initially static, it must hold that $\Delta K_I>0$ where $\Delta K_I$ is the change in kinetic energy of the ice. Further,
\begin{equation}
\label{eq:DKsys}
    \Delta K_{\text{sys}} = \Delta K_I - \Delta K_S.
\end{equation}
We treat the ice as having a velocity post-collision equal to $V_{\text{eq}}$ -- a reasonable approximation if $m_S/m_I$ is large. Therefore, $\Delta K_I=0.5 m_IV_{\text{eq}}^2=0.5m_I(v_S\cos(\theta))^2$ and from \eqref{eq:DeltaKsys}, \eqref{eq:DKsys}:
\begin{equation}
\label{eq:KE1}
    \begin{split}
        |\Delta K_S| = \frac{m_I^2}{2(m_S + m_I)}\big(v_S\cos(\theta)\big)^2.
    \end{split}
\end{equation}
From Figure \ref{fig:collision}(a), observe that $\sin(\theta)=d/r$ where $d$ is the lateral distance between the collision point and the center of the ice and $r$ is the radius of the ice. Therefore from \eqref{eq:KE1}, we may write $|\Delta K_S|$ as a function of $d,r, m_I, v_S$:
\begin{equation}
\begin{split}
       & \Delta K_S(d,r, m_I, v_S) = \frac{1}{2} \frac{m_{I}^2}{m_{S} + m_{I}}\left [v_S \cos\left(\arcsin\left(\frac{d}{r}\right)\right) \right ]^2\\
        & = \frac{v_S^2m_I^2}{2(m_S+m_I)}\left(\frac{r^2-d^2}{r^2}\right), d \in [0, r].
        \label{eq:KE2}
        \raisetag{1.5\normalbaselineskip}
\end{split}
\end{equation}
Observe that the function $\Delta K_S$ is large if collisions are head-on ($\theta=d = 0$) or if $m_I/m_S$ is large. Thus $\Delta K_S$ has the benefit of effectively penalizing collisions with ice that are head-on and or involve large ice relative to the ship. Using $\Delta K_S$, we describe the cost function $u$ from \eqref{eq:cost_fn}.

\begin{figure}[!t]
    \centering
    \includegraphics[scale=0.25]{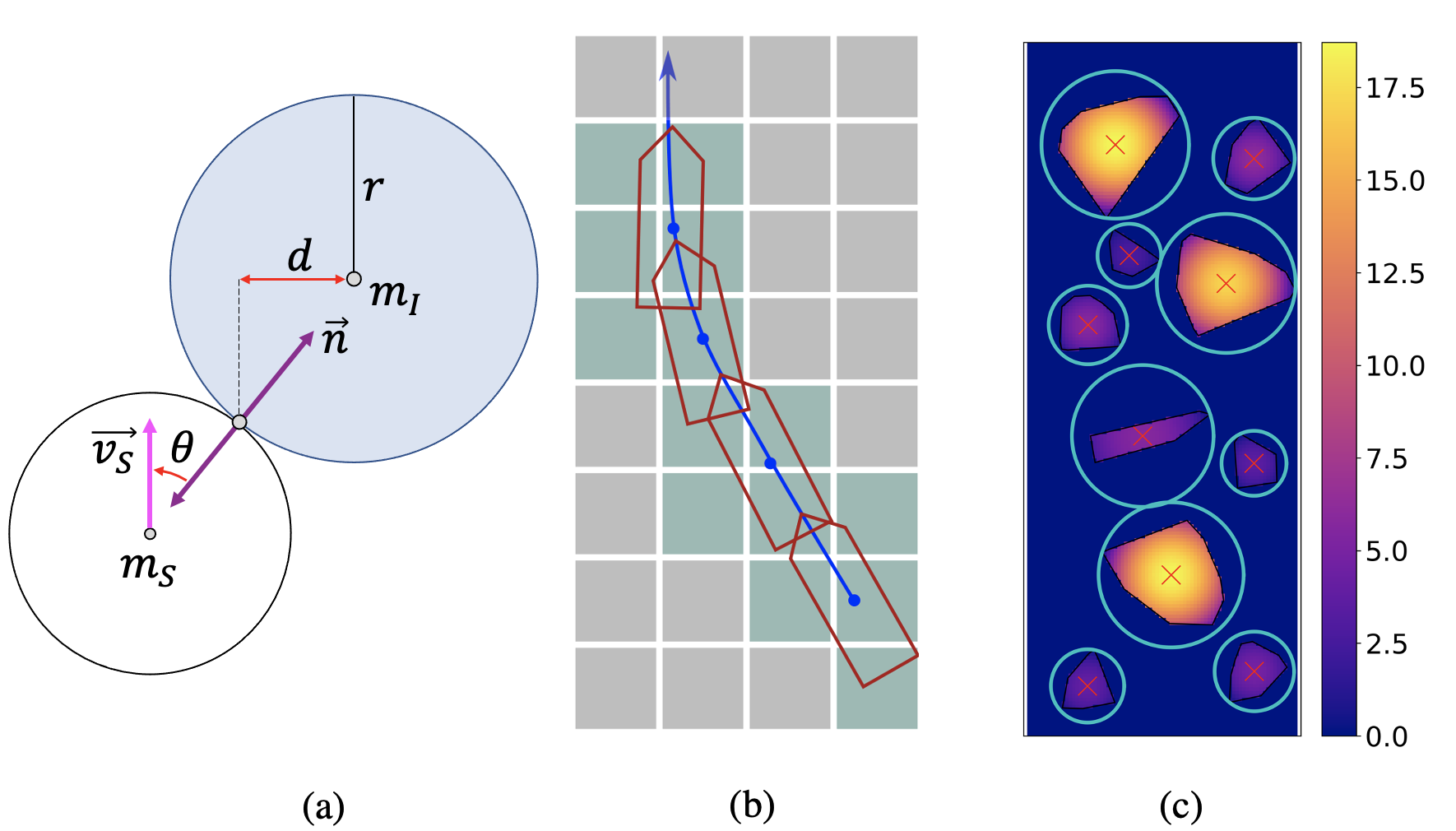}
    \caption{(a) Collision scenario between two disk-shaped bodies. (b) Example path (blue) with ship footprint (red) at regular intervals of arc length with exaggerated costmap resolution to illustrate the swath (green). (c) Sample costmap where the color bar indicates the cost. Obstacle centroids are in red and bounding circles are in light blue.}
    \label{fig:collision}
\end{figure}

To efficiently compute collision costs across an ice channel $\mathcal{C}$, we use a costmap representation of the planar environment as done in~\cite{pivtoraiko2009differentially}. In particular, the channel $\mathcal{C}\subseteq\mathbb{R}^2$ is discretized into a square grid where each grid cell is assigned an identifying tuple $k\in\mathbb{N}^2$ corresponding to the cells' center and a cost $c_{obs}(k, \obs, v_S)$ given obstacles $\obs$ (determined in Line 2 of Algorithm \ref{alg:GeneralPlanner}), and $v_S$ (Line 3). Each cell is occupied by at most one obstacle. The set of costmap cells that are occupied from the area of a ship as it traverses a path $\pi$ is called the \emph{swath} of $\pi$ and the area occupied by the ship is called the \emph{footprint}~\cite{pivtoraiko2009differentially} (see Figure \ref{fig:collision}(b)).

For our planning task, the costmap is effectively a lookup table for computing the collision cost associated with a particular path between a pair of lattice vertices, given an appropriate mapping from the configuration space to the costmap. For each polygonal obstacle $O_j\in\obs$, let $C_j$ denote the position of the centroid of $O_j$, $r_j$ the radius of its bounding circle centered at $C_j$ (Figure \ref{fig:collision}(c)), and $m_j$ its mass. The function $c_{obs}(k, \obs, v_S)$ for a cell $k$ and current ship speed $v_S$ uses the ship kinetic energy loss from (\ref{eq:KE2}):
\begin{equation}
\label{eq:cdef}
   c_{\text{obs}}(k, \obs, v_S) = \begin{cases}
    \Delta K_S(q, r_j, m_j, v_S), \ & k\cap O_j\neq \emptyset\\
    0  \ &\text{otherwise},
    \end{cases}
    \raisetag{3\normalbaselineskip}
\end{equation}
where $q=||k-C_j||$, the Euclidean distance from the cell to the centroid of the obstacle. This function is well defined since each cell can be occupied by at most one obstacle. A sample costmap is illustrated in Figure \ref{fig:collision}(c). 

Finally, with $c_{\text{obs}}(k,\obs, v_S)$ given in \eqref{eq:cdef} we define the cost $u(\pi, v_S)$ for a candidate path $\pi$ (and ships current speed $v_S$) with final arc-length $s_f$ using the discretization of \eqref{eq:cost_fn}:
\begin{equation}
    \label{eq:discrete_cost}
    u(\pi, v_S) = s_f + \alpha\sum_{k\in \text{swath of } \pi} c_{\text{obs}}(k,\obs, v_S).
\end{equation}

\subsection{Heuristic}
\label{sec:heuristic}
Given the current ship pose $P$ (Line 3), we present an admissible heuristic $h$ with a closed form to improve the runtime of the path planning step in Line 6.

At a high-level, we compute the shortest path from $P$ to an infinite line $\mathcal{G}_{\infty}$ that is colinear with the intermediate goal line $\mathcal{G}^i$ (Line 4), subject to the unicycle model described in Section \ref{sec:Model}. This heuristic is admissible to our graph search implementation since the cost function $u$ is lower bounded by path length and $\G^i\subset \G_{\infty}$ (thus the shortest path to $\G_{\infty}$ is no longer than the shortest path to $\G^i$). To characterize the proposed heuristic, we offer the following result:

\begin{theorem}[Closed-form Heuristic]
The shortest path from $P = (P_x, P_y, P_\theta)$ to the infinite line $\G_{\infty}$ with minimum turning radius $r_{\min}$ is a Dubin's path. Referencing Figure~\ref{fig:heuristic_ab}, 
\begin{enumerate}
    \item if $\G_{\infty}$ lies above the point $o$, the path is of the form CS (circular arc C of radius $r_{\min}$, followed by a straight line S) and S intersects $\G_{\infty}$ at a right angle (Fig.~\ref{fig:heuristic_ab}(a));
    \item otherwise, the path is of the form C (Fig.~\ref{fig:heuristic_ab}(b)).
\end{enumerate}
Given these two cases, $h(P)$ is the path length and is given analytically by:
\begin{equation*}
    \begin{split}
    h(P) &=\begin{cases}
    h_1(P), \ &\text{if } o_y \leq \G_y \ \text{(case 1)},\\
    h_2(P), \ &\text{otherwise (case 2)}
    \end{cases},\\
    \text{where} &\\
         h_1(P) &= r_{\min}\min\left(\left|P_{\theta} - \frac{\pi}{2}\right|,\left|P_{\theta} - \frac{5\pi}{2}\right| \right) + \G_y - o_y\\
         h_2(P) &= r_{\min}\left|P_{\theta} - m\arccos\left(\frac{o_y - \G_y}{r_{\min}}\right) - n\right|\\
         o_y &= P_y + m r_{\min} \cos(P_{\theta})\\
         m &= \begin{cases}
        +1, \ &\text{if} \ P_{\theta}\in[0, \frac{\pi}{2}]\cup[ \frac{3\pi}{2}, 2\pi]\\
        -1, \ &\text{otherwise}
        \end{cases}\\
        n &= \begin{cases}
        0, \ &\text{if} \ P_{\theta}\in[0, \frac{\pi}{2}]\\
        \pi, \ &\text{if} \ P_{\theta}\in(\frac{\pi}{2}, \frac{3\pi}{2}]\\
        2\pi, \ &\text{if} \ P_{\theta}\in(\frac{3\pi}{2}, 2\pi]
        \end{cases}\\
    \end{split}
\end{equation*}\\
and where $\G_y$ is the constant $y$-value of the goal $\G_\infty$.
\end{theorem}

\begin{proof}
Without loss of generality, we consider a turn at an angle of $P_\theta \in [\frac{\pi}{2}, \pi]$ as in Figure \ref{fig:heuristic_ab} (a similar analysis can be made for angles in the other three quadrants).

\emph{Case 1:}
Suppose that $\G_\infty$ lies above $o$. This is equivalent to the condition that the $y$-coordinate of $o$, is no more than the $y$-coordinate of $\G_\infty$, i.e., $o_y = P_y+r_{\min}(-\cos(P_{\theta}))\leq \G_y$. Let $P^g = (P^g_x, P^g_y)$ be any point on $\G_{\infty}$ and is outside of the circle of radius $r_{\min}$ centered at $o$ (it is trivial to show $P^g$ cannot be inside this circle for the shortest path from $P$ to $\G_{\infty}$). Since the heading of $P^g$ is not specified, the shortest path from $P$ to $P^g$ is of the form CS~\cite{bui1994shortest,ma2006receding} where $S$ intersects $\G_{\infty}$ at an angle $\theta$. Therefore, determining the shortest path from $P$ to $\G_{\infty}$ reduces to computing
\begin{equation}
    \min_{\theta} \quad L(C)+L(S),
\end{equation} where $L(\cdot)$ denotes length. We observe 
\begin{equation}
\begin{split}
L(C) &= r_{\min}(P_\theta - \theta),\\
L(S) &=  \frac{\mathcal{G}_y - [P_y - r_{\min}(\cos(P_\theta) - \cos(\theta))]}{\sin(\theta)}.
\end{split}
\end{equation}
Further, we observe that the total path length $L(C)+L(S)$ is minimized by $\theta=\pi/2$. Replacing this value in $h_1(P)=L(C)+L(S)$ yields the result of the Theorem for the case $o_y\leq \G_y$.

\emph{Case 2:} Suppose instead that $\G_\infty$ lies below $o$. In this case, the shortest path from $P$ to $\G_\infty$ consists only of a circular arc $C$ with minimal length equal to the result of the Theorem for the case that $o_y > \G_y$.
\end{proof}
 In the next section, we detail the integration of our proposed planner with the experimental platform.

\begin{figure}[!t]
    \centering
    \includegraphics[scale=0.2]{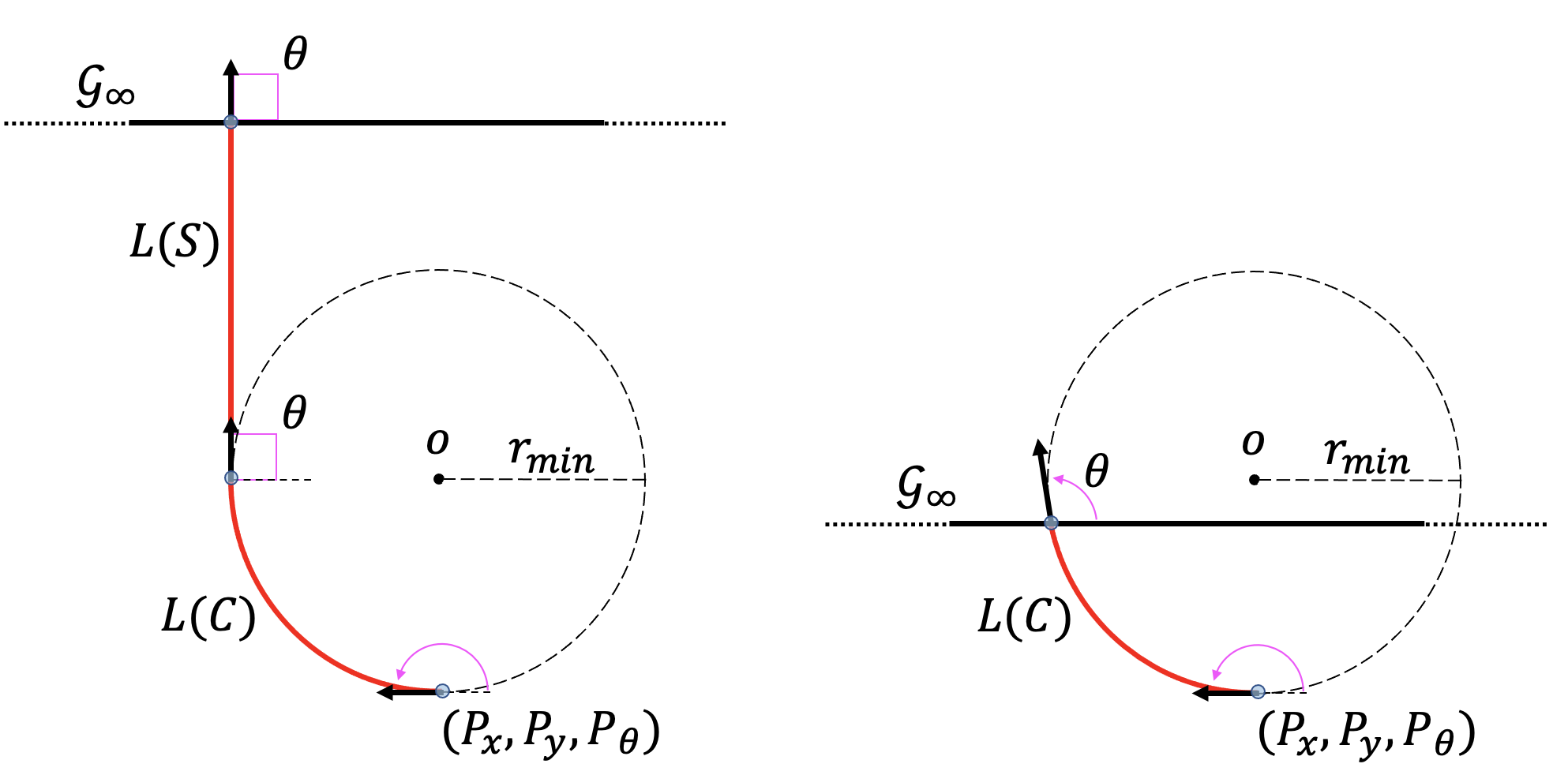}
    \caption{Diagram depicting the geometry of the Dubins path to an infinite line $\G_{\infty}$ for two possible cases.}
    \label{fig:heuristic_ab}
\end{figure}

\section{Details on Experimental Platform}
We validated our proposed approach by integrating our navigation framework with the NRC ice tank facility shown in Figure \ref{fig:icetank}, which is complete with a model vessel and real-time overhead vision system \cite{gash2020machine}.

\emph{Environment and Physical Model:} The NRC ice basin is $12\ \text{m}  \times  76\ \text{m}$ with ice thickness up to 200 mm. We used the same 1:45 scale platform supply vessel (PSV) model deployed in \cite{murrant2021dynamic} shown in Figure \ref{fig:icetank}. 

\emph{Vision System:} The facility contains 20 ceiling cameras sending ice information at a frequency of 1 Hz. The ASV configuration is computed at 50 Hz using the tracking system and an on-board inertial measurement unit (IMU) \cite{gash2020machine, murrant2021dynamic}. Given the two update rates, we set $\Delta t = 1$ in Algorithm \ref{alg:GeneralPlanner}.

\emph{Controller:}
\label{control section}
We employed a Dynamic Positioning (DP) controller -- widely used in marine navigation~\cite{ahani2020optimum, mehrzadi2020review} -- which generates thruster/propeller commands to regulate position and heading. 
We used a constant nominal velocity of $0.3$ m/s and a minimum turning radius $r_{\min}=2$ m, computed according to a full-scale vessel.

\emph{Planner Parameters:} Discretization was set to $1 \times 1$ m for planar position and $8$ equally spaced values for heading. In Figure \ref{fig:prims} we show the two sets of different motion primitives generated for the axis-aligned and non-axis aligned directions which consist of $15$ and $19$ primitives respectively.

Costmap grid resolution was $0.25$ m, finite horizon distance $\Delta h = 20$ m in Algorithm \ref{alg:GeneralPlanner}, and the cost function tuning parameter $\alpha = 10$ in \eqref{eq:cost_fn}.

\begin{figure}[!t]
    \centering
    \includegraphics[width=0.85\linewidth]{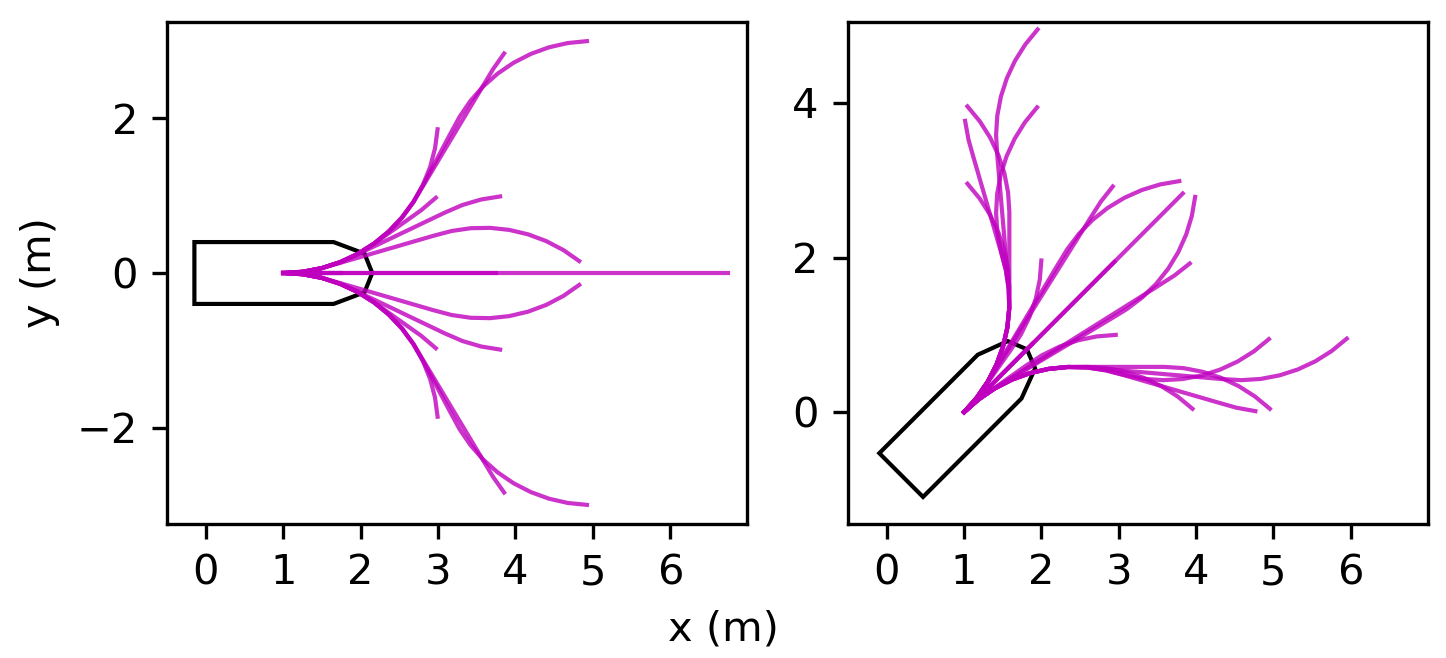}
    \caption{Axis aligned (left) and non-axis aligned (right) motion primitives generated for a 1:45 scale ASV model.}
    \label{fig:prims}
\end{figure}

\section{Results}
\label{resultssection}
We demonstrate the efficacy of the proposed approach through simulation and physical experiments. Each trial consisted of navigating a ship across the length of an ice channel to the specified goal $\G$. Our simulations were done in Python and the 2D physics library Pymunk \cite{pymunk} was used as our backbone for the simulation experiments. 

\subsection{Simulation Setup}
The simulation setup matches that of the experimental platform (e.g., NRC ice tank dimensions, controller, 1:45 vessel model). Given a specified target ice concentration (expressed as a value in [0, 1] where 1 is maximal ice cover), we populated the map with randomly generated non-overlapping polygons subject to the following constraints: $R_{\min} = 0.5\text{m}, \ R_{\max} = 2\text{m}, \ y_{\min} = 5\text{m}, \ y_{\max} = 70\text{m},$ where $R_{\min}, R_{\max}$ are constraints on the radii of the primal circles from which the polygons stem from while the $y$ constraints enforce the polygon vertices to be within a specified region of the environment. We randomized the ship starting $x$ position, and fixed the ship starting $y$ position and heading to $2$ m and $\pi / 2$ respectively. The goal line segment $\G$ was set to $72$ m in the $y$-axis. 


\subsection{Simulation Results}
In simulation, we explored our framework's effectiveness as a function of ice concentration via comparisons to Algorithm \ref{alg:GeneralPlanner} using other planning schemes (Line 6). Two baseline planning algorithms were considered which we refer to as \emph{straight} and \emph{skeleton}. The former is simply a planner that returns a constant straight path from the ship's current position to the goal $\G$ and the latter refers to the shortest open-water path routing approach described in \cite{gash2020machine} and \cite{murrant2021dynamic}. Their approach constructs morphological skeletons based on the ice environment to generate paths. We refer to the three different versions of our navigation framework based on their planning scheme, i.e. \emph{lattice}, \emph{straight}, and \emph{skeleton}. We performed the same set of trials across all three of these navigators. In total we ran $3$ navigators $\times$ $50$ trials $\times$ $4$ ice concentrations $= 800$ experiments. The ice concentrations considered were $0.2, \ 0.3, \ 0.4,$ and $0.5$. Note, we used a 1st order Nomoto model \cite{fossen2011handbook} to describe our vessel dynamics in simulation.

\subsubsection{Metrics}
We computed a running total of the kinetic energy lost by the ship $\Delta K_S$ due to collisions with ice, using the physics simulator. To better interpret this metric, we used the ship total kinetic energy loss in the straight navigator to normalize the values for each trial for lattice and skeleton. Further, we captured the average tracking error across all simulations of the lattice and skeleton navigators to gauge how easy each reference path was to track. Finally, we logged the mass of the ice collided with for each collision that occurred in each simulation and obtained a probability density function for each navigator. This effectively approximates the probability of colliding with an ice floe of any size for each navigator.  

\subsubsection{Results}
We present two figures that capture the main results from our simulations. Figure \ref{fig:violin} illustrates the improved performance of our approach (lattice) with regards to the total kinetic energy lost by the ship from collisions with ice. Our mean and median are consistently lower across all 4 ice concentrations. For the trials done in the most dense ice fields (i.e. ice concentration $= 0.5$) we achieved a mean of $0.39$ vs. $0.43$ for skeleton. In other words, our approach achieved a $9\%$ decrease in terms of kinetic energy lost. In addition, our maximum (outliers not shown in plot) is significantly lower both in the $0.2$ and the $0.5$ ice concentrations, e.g., for $0.5$ concentration our max was $1.0$ and skeleton was $1.6$.

In Figure \ref{fig:pdf}, we show the probability density functions for collisions across different navigators 
\cite{daley2014gpu}. In the 0.4 ice concentration scenario, the mean mass of ice colliding with the ship was $2.0$ for the proposed navigator, $2.6$ for the skeleton, and $2.8$ for straight while in the $0.5$ ice concentration scenario, $1.4$ for the proposed, $1.9$ for the skeleton, and $2$ for the straight.  Therefore, on average the proposed approach collided with ice that was $24\%$ smaller in mass than the skeleton navigator, and $29\%$ smaller than the straight.

Further, the tracking error for paths using the proposed approach was, on average $50\%$ lower than that of the skeleton navigator. Note, path planning with our approach took an average of $90$ ms with a max of $127$ ms which are both comparable to the skeleton navigator. Finally, graph search using our heuristic from \ref{sec:heuristic} expanded an average of $15\%$ (max of $48\%$) fewer nodes than a euclidean distance baseline.

\begin{figure}[t]
    \centering
    \includegraphics[scale=0.48]{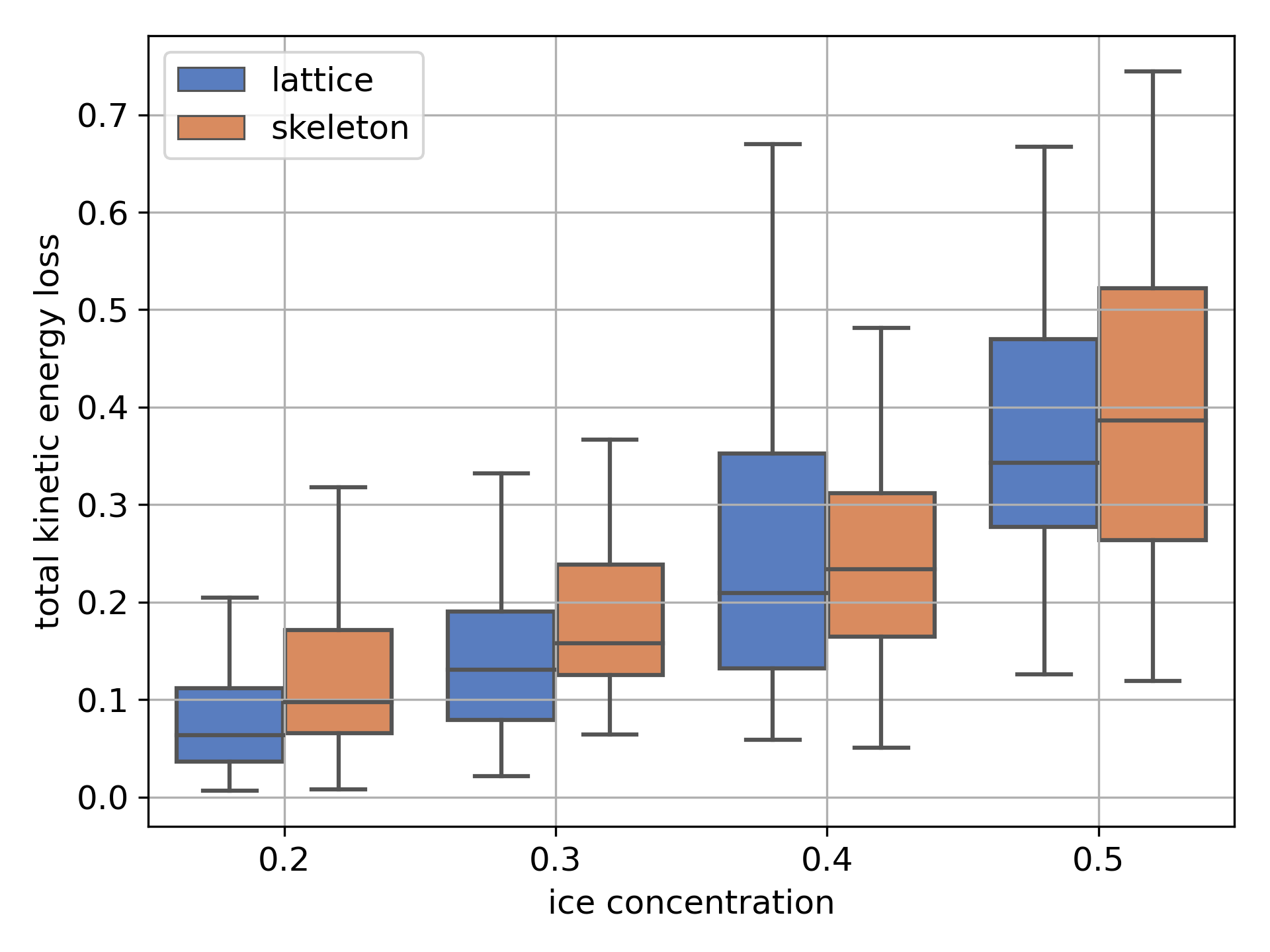}
    \caption{Total kinetic energy loss by ship (normalized by straight navigator) as a function of the ice concentration for our navigation framework using lattice planning (ours) and skeleton.}
    \label{fig:violin}
\end{figure}

\begin{figure}[t]
    \centering
    \includegraphics[width=\linewidth]{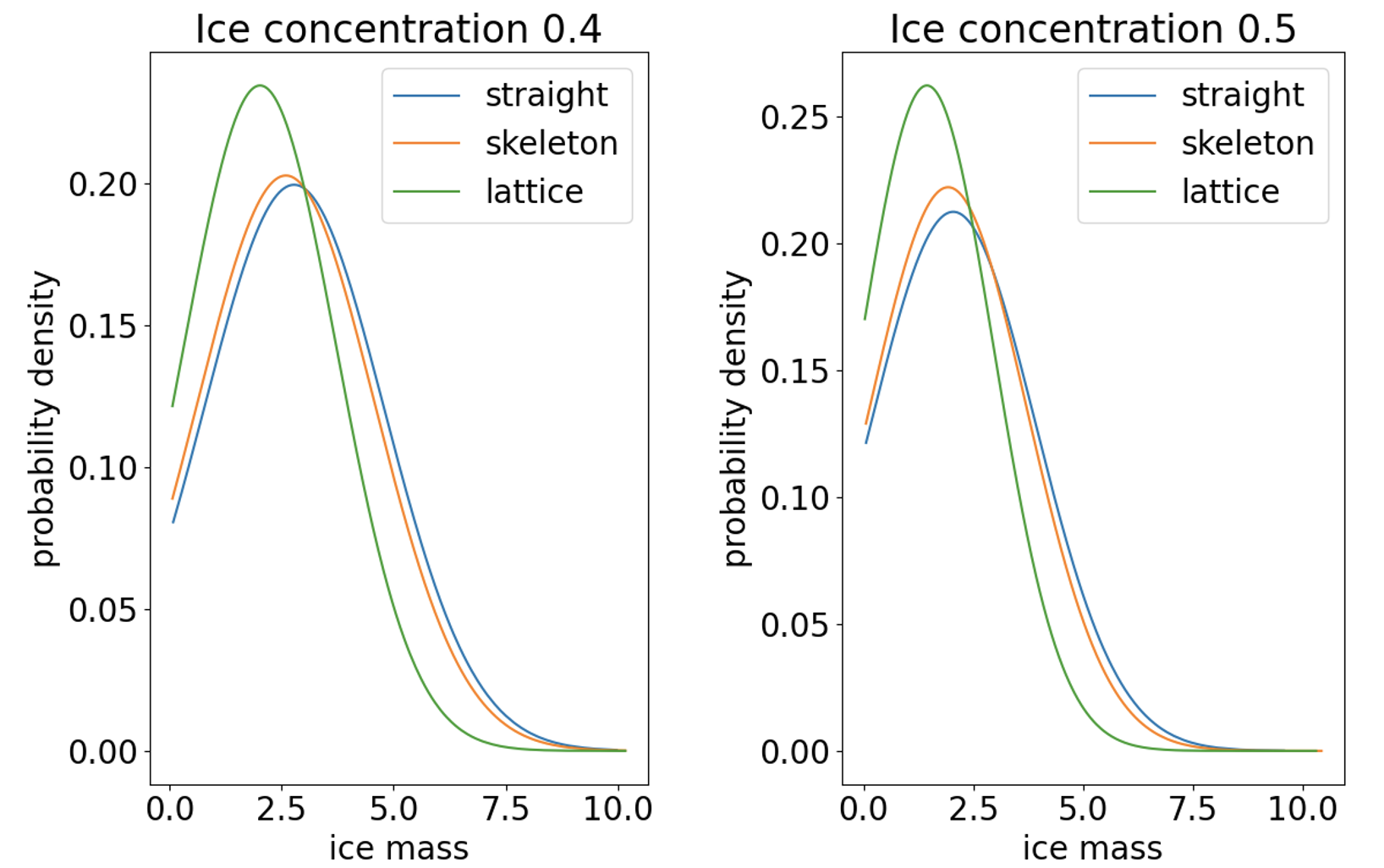}
    \caption{Frequency of ice collision by mass for the three planning methods, smoothed using kernel density estimation. Note, ice mass is dimensionless here.}
    \label{fig:pdf}
\end{figure}

\begin{figure}[!h]
    \centering
    \includegraphics[width=0.55\linewidth]{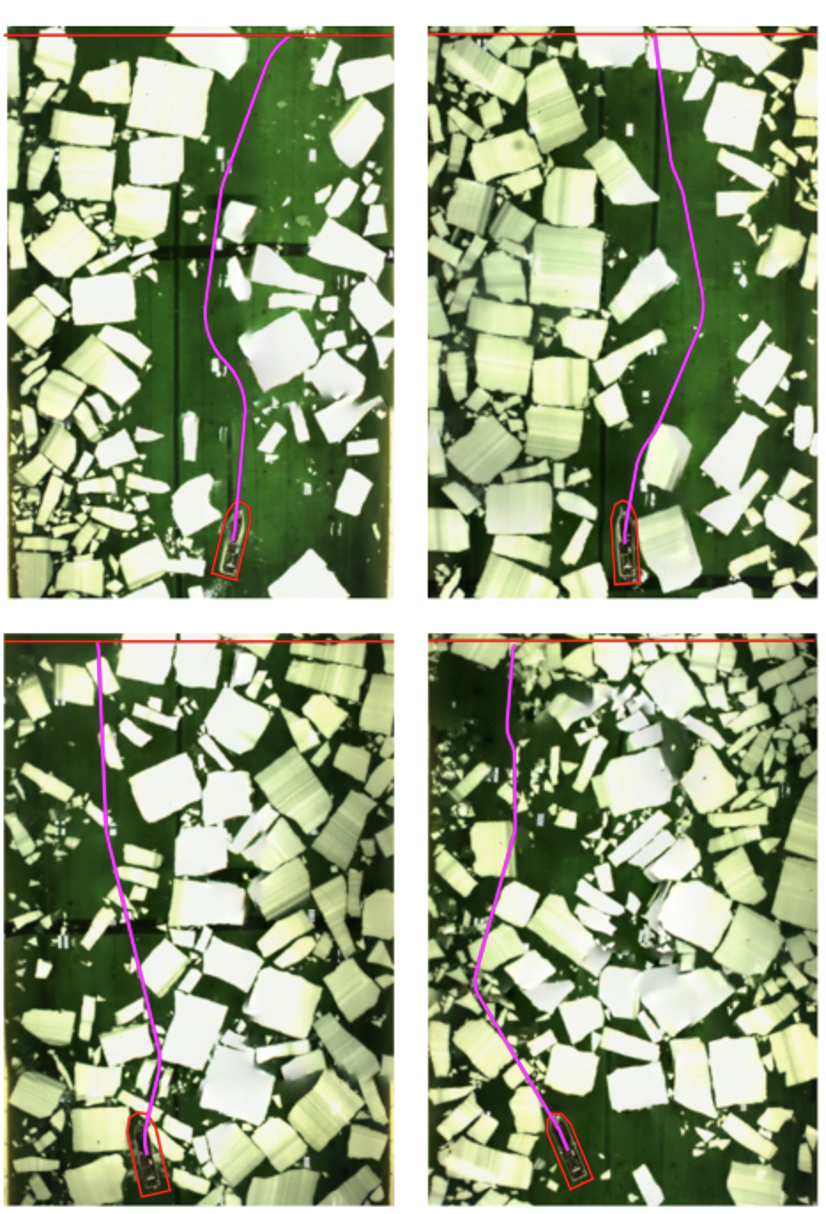}
    \caption{Snapshots of four different trials done in the $12\ \text{m}  \times  76\ \text{m}$ ice tank at varying levels of ice concentration.}
    \label{fig:allships}
\end{figure}

\subsection{Real World Results}
We conclude with an overview of the experiments ran with our planner in the physical NRC ice tank. We ran a total of 8 trials across two different ice concentrations (medium and high). This series of tests proved to be the first successful attempts at fully autonomous navigation across the entire length of the ice basin in the NRC ice tank facility following partial successes in preliminary testing done in \cite{murrant2021dynamic}.

We show four representative snapshots (Figure \ref{fig:allships}) taken during the experiments and make a series of observations that highlight our planner features. The advantage of having a non-distinct goal somewhere along a line segment is clearly shown across each snapshot. We also see the smoothness of the turn-constrained planned paths that result in better tracking behavior than any-angle approaches such as the skeleton planner. Most importantly, the collisions that occur along the paths are consistent with our design decisions captured by our proposed cost function (\ref{eq:discrete_cost}). Namely, avoiding larger ice floes over smaller ones, and colliding with ice such that the ship loses minimal kinetic energy. These last two points are most apparent in the 2 snapshots taken from the high ice concentration experiments (bottom left and bottom right in Figure \ref{fig:allships}).

\section{Conclusions}
In this work, we proposed an autonomous real-time navigation framework for ASVs through ice-covered waters. Our method tailored well-known lattice-based path planning and receding horizon-based planning to produce an effective navigation strategy for this environment. A key component of our framework is the proposed cost function, which captures the energy lost by the ship during a collision and heavily penalizes head-on collisions with larger ice floes. Our planner achieved better overall performance than existing planning solutions designed for navigation in icy waters. In future work, we intend to incorporate a predictive ice-motion model into our planner.









\bibliographystyle{IEEEtran}
\bibliography{bib.bib}

\end{document}